\documentclass[wcp]{jmlr}

\usepackage{amssymb,amsmath}
\usepackage{float}
\usepackage{xcolor}
\usepackage{graphicx}
\usepackage{wrapfig, blindtext}
\usepackage{ulem}
\usepackage{mathtools}
 
\newcommand\numberthis{\addtocounter{equation}{1}\tag{\theequation}}

\usepackage{booktabs}
\usepackage{multirow}
\usepackage{siunitx}
\usepackage{hyperref}
\usepackage{amsfonts}
\usepackage{amssymb}

\usepackage{balance}

\usepackage{placeins}

\usepackage{verbatim}
\usepackage{microtype}

\usepackage{booktabs} 

\usepackage{hyperref}

\usepackage{paralist}
\usepackage{float}
\usepackage{wrapfig}
\usepackage{caption}
\usepackage{multirow}
\usepackage{algorithm,algorithmic}
 
\usepackage{textcomp}

\newtheorem{thm}{Theorem}

\newtheorem{lem}[thm]{Lemma}

\def\R{\mathbb{R}}
\def\suchthat{\;:\;}
\def\given{\;|\;}
\def\opt{\text{OPT}}
\newcommand{\norm}[1]{\left\|#1\right\|}
\newcommand{\sqd}[1]{\left\|#1\right\|^{2}}
\newcommand{\prob}[1]{{\sf Pr} \left(#1\right)}
\newcommand{\expec}[1]{{\sf E} \left[#1\right]}
\newcommand{\size}[1]{\left|#1\right|}
\newcommand{\eps}{\epsilon}
\newcommand{\poly}{\mathrm{poly}}
\newcommand{\kmpp}{\mathrm{KM++}}
\newcommand{\rand}{\mathrm{RAND}}
\newcommand{\tkmpp}{\mathrm{TKM++}}
\newcommand{\rkmpp}{\mathrm{RKM++}}

\usepackage{breakcites}
\usepackage[hyphenbreaks]{breakurl}

\title{Improved Outlier  Robust  Seeding for $k$-means}
\author{
\Name{Amit Deshpande }  \Email{amitdesh@microsoft.com}\\
  \addr Microsoft Research, India
  \AND
\Name{Rameshwar Pratap} \Email{rameshwar@cse.iith.ac.in}\\
  \addr  IIT Hyderabad, India
}

\begin{document}

\maketitle

\begin{abstract}
The $k$-means is a popular clustering objective, although it is inherently non-robust and sensitive to outliers. Its popular seeding or initialization called $k$-means++ uses $D^{2}$ sampling and comes with a provable $O(\log k)$ approximation guarantee \cite{AV2007}. However, in the presence of adversarial noise or outliers, $D^{2}$ sampling is more likely to pick centers from distant outliers instead of inlier clusters, and therefore its approximation guarantees \textit{w.r.t.} $k$-means solution on inliers,  does not hold.

Assuming that the outliers constitute a constant fraction of the given data, we propose a simple variant in the  $D^2$ sampling distribution, which makes it robust to the outliers. Our algorithm runs in $O(ndk)$ time, outputs $O(k)$ clusters, discards marginally more points than the optimal number of outliers, and comes with a provable $O(1)$ approximation guarantee.
 
 Our algorithm can also be modified to output exactly $k$ clusters instead of $O(k)$ clusters, while keeping its running time linear in $n$ and $d$. This is an improvement over previous results for robust $k$-means based on LP relaxation and rounding \cite{Charikar}, \cite{KrishnaswamyLS18} and \textit{robust $k$-means++}  \cite{DeshpandeKP20}. Our empirical results show the advantage of our algorithm over $k$-means++~\cite{AV2007}, uniform random seeding, greedy sampling for $k$-means~\cite{tkmeanspp}, and \textit{robust} $k$-means++~\cite{DeshpandeKP20}, on standard real-world and synthetic data sets used in previous work.  
Our proposal is easily amenable to scalable, faster, parallel implementations of $k$-means++ \cite{Bahmani,BachemL017} and is of independent interest for coreset constructions in the presence of outliers \cite{feldman2007ptas,langberg2010universal,feldman2011unified}.
\end{abstract}

\section{Introduction}
The $k$-means clustering is a popular tool in data analysis and an important objective in statistics, data mining, unsupervised learning, and computational geometry. 
The objective of the $k$-means clustering is to find $k$ centers that minimize the sum of squared distances of all the points to their nearest centers, respectively. Given a set $X \subseteq {\R}^{d}$ of $n$ data points and an integer $k>0$, the $k$-means objective is to find a set $C \subseteq {\R}^{d}$ of $k$ centers that minimizes
\begin{align*}
\phi_{X}(C) = \sum_{x \in X} \min_{c \in C} \sqd{x-c}.   \numberthis \label{eq:k-means_objective} 
\end{align*}

Finding optimal solution of the $k$-means objective stated in Equation~\ref{eq:k-means_objective} in NP-Hard~\cite{AloiseDHP09}. This problem is NP-Hard even for a restricted instance when all points are in the plane~\cite{MahajanNV12}. However, several efficient approximation algorithms and heuristics have been developed to address this. 
The most popular algorithm for the $k$-means remains Lloyd's $k$-means method \cite{Lloyd}, which is a simple, fast heuristic that starts with any given initial solution and iteratively converges to a locally optimal solution. 

Although the $k$-means problem is well-studied,   the algorithms developed for the problem can perform poorly on real-world data. The reason is that the real-world datasets contain outliers, and $k$-means objective functions and algorithms are extremely sensitive to outliers.  Outliers can drastically change the quality of the clustering solution, and therefore, it is important to consider them while designing algorithms for the $k$-means objective. We state the objective function of  \emph{$k$-means with outliers} as follows:
  given a set $X \subseteq \R^{d}$ of $n$ points, an integer $k > 0$, and the number of outliers   $z$, the objective of the \emph{ $k$-means with outliers} is to find a set $C \subseteq \R^{d}$ of $k$ centers that minimizes
\begin{align*}
    \rho_{X}(C) &= \min_{\substack{Y \subseteq X \\ \size{Y} = n-z}}~ \sum_{x \in Y} \min_{c \in C} \sqd{x - c}. \numberthis \label{eq:robust_cost}
\end{align*}
\textit{ \noindent\textbf{Problem statement:} In this work, we aim to design efficient and practical approximation algorithms for $k$-means clustering with outliers problem, which solves the optimization problem stated in Equation~\ref{eq:robust_cost}.}

\noindent \paragraph{$D^2$-sampling for $k$-means:} The $k$-means++ or $D^2$-sampling~\cite{AV2007} suggests an adaptive sampling algorithm for $k$-means problem (stated in Equation~\eqref{eq:k-means_objective}). In this sampling approach, the first point is sampled uniformly at random from the given points, and the sampled point is designated as a cluster center. Then the second point is sampled with probability proportional to its squared distance from the first center, and designated as the second cluster center. In general, in each step, a new point is sampled with probability proportional to its square distance to the nearest cluster center sampled so far. If we sample $k$ cluster centers following this distribution, the clustering obtained is $O(\log k)$-approximation to the global optimum, in expectation. However, a limitation of $D^2$ sampling distribution is that it is extremely sensitive to outliers. Consider a scenario where $99\%$ of points are well clustered and the remaining $1\%$ of points are very far away from these clusters. The $ D^2$ sampling on this dataset is likely to pick outliers as cluster centers, and the final clustering results obtained is likely to be very far away from the optimal clustering. In this work we propose a simple tweak to $D^2$-sampling, making it robust to outliers. 

\section{Our results}

We propose a simple initialization for \emph{$k$-means with outlier} objective (stated in Equation~\ref{eq:robust_cost})  using a simple modification of $D^{2}$ sampling. Our algorithm runs in $O(ndk)$ time and gives $O(\log k)$ approximation guarantee using $k$ clusters and $O(1)$ bi-criteria approximation guarantee using $O(k)$ clusters. Both these algorithms can be made to output exactly the same number of outliers as the optimal solution, unlike previous algorithms that need to discard extra points than the optimal number of outliers. The pseudocode of our algorithm is presented in Algorithm \ref{algo:robust}, and the precise statement of our approximation guarantee appears in Theorems \ref{thm:robust}, \ref{thm:log-k}. In addition, Table \ref{tab:comparison} shows a detailed comparison of our results with previous work.

 
We perform extensive experiments to compare  our proposal  with several baselines and summarise it in Section \ref{sec:exp}. We use synthetic and real-world datasets for our experimentation. In synthetic datasets, we synthetically generated inlier clusters and outliers. In real-world datasets, we consider two scenarios: a) where we consider small clusters as outliers and b) where we randomly sample a small fraction of points, and add large Gaussian noise to their features and consider them as outliers.  We used the following four baselines for empirical comparison -- random seeding, $k$-means++~\cite{AV2007}, ~\cite{DeshpandeKP20}, and ~\cite{tkmeanspp}. Random seeding and $k$-means++~\cite{AV2007}  do not solve $k$-means with outlier problems. We use them as heuristics, where $k$ points are sampled using their respective sampling distribution, the farthest few points are marked as outliers, and the remaining points as inliers on which we compute our evaluation metrics (precision/recall/clustering cost).  Our empirical findings are as follows:  We outperformed on most instances in both the evaluation metric --  precision/recall,  and the clustering cost (Equation~\eqref{eq:robust_cost}). Note that \texttt{random} and $k$-means++~\cite{AV2007} are among the fastest. However, their performance on precision/recall and clustering cost metrics are significantly worse on almost all the datasets. The running time of our proposal is observed to be faster than~\cite{DeshpandeKP20}, whereas it remains comparable \textit{w.r.t.}~\cite{tkmeanspp}. To summarise, our proposal gives both a) theoretical guarantee on the clustering cost \textit{w.r.t.} the optimal solution, and b) strong empirical performance by suggesting a faster and more accurate algorithm for $k$-means with outliers problem.

\section{Related work}

Outlier detection is a well-studied problem in data mining and statistics~\cite{DBLP:journals/csur/ChandolaBK09}. Intuitively, outliers are data points far from their nearest neighbours.  One classical approach to detecting outliers is via the $k$-Nearest Neighbour algorithm. In statistics, Mahalanobis Distance and Minimum Covariance Determinant (MCD)~\cite{DBLP:journals/technometrics/RousseeuwD99}  are a couple of notable approaches for outliers detection. Our problem statement is completely different from the one mentioned here. These methods focus on detecting outliers related to a distribution/density function, whereas we aim to produce clustering that is near-optimal in terms of an objective that is defined solely on inliers (stated in Equation~\eqref{eq:robust_cost}). Another practical heuristic is to first identify and discard outliers by running these outliers detection algorithms and then run $k$-means clustering on the remaining points. However, this heuristic doesn’t provide any theoretical guarantee on its clustering cost \textit{w.r.t} optimal clustering over inliers. In what follows, we state some notable results for $k$-means with outliers problem, and we attempt to compare and contrast our results with them as follows.

\subsection{Impractical algorithms with theoretical guarantee}
Charikar \textit{et. al.}~\cite{Charikar} suggests a clever modification of $D^2$-sampling~\cite{AV2007} and gives a $3$-approximation for the robust $k$-center problem with outliers. For 
 $k$-median with outliers, they give a $4(1 + 1/\epsilon)$-approximation in polynomial time that discards at most $(1+\epsilon)$ times more outliers than the optimal solution. Their approximation depends on the number of extra points deleted as outliers, whereas the approximation of our approach is independent of that.   Chen~\cite{chen2008constant} suggest a polynomial time constant factor approximation to $k$-means and $k$-median with outliers that do not discard extra outliers. However, the algorithms are not designed to be practical.
 
\cite{KrishnaswamyLS18} give (roughly) $53$-approximation algorithm for $k$-means with $z$ outliers, and output exactly $k$ centers and doesn’t discard any extra outliers. They use LP relaxation and an iterative rounding approach. \cite{FriggstadKRS18} give $(1+\epsilon)$-approximation to $k$-means with $z$ outliers in a fixed   dimensional
euclidean space. They use local search and output $(1+\epsilon)k$ centers and discard exactly $z$ outliers. However, both of these algorithms are not designed to be practical. 
There are few sampling-based robust coreset construction techniques~\cite{feldman2007ptas,feldman2011unified} which construct small-size coresets and give $(1+\eps)$-approximation to various robust versions of clustering problems. These algorithms can be considered extensions of sampling-based techniques focusing on getting close to optimal solutions in polynomial time.  
 
 Chen~\cite{chen2008constant} gives an efficient constant factor approximation to $k$-means and $k$-median with outliers, and their result doesn’t discard any extra outliers.  However, their algorithm is not designed to be practical compared to $k$-means++ and its variants. The algorithm of Kumar et al. \cite{KumarSS2010} suggest $(1+\epsilon)$-approximation algorithm for $k$-means with outliers for the datasets when clusters are balanced and their algorithm discards a slightly larger fraction of the outliers than the optimal solution. However, their algorithm is impractical and has an exponential running time in $k$.

\subsection{Practical heuristic without theoretical guarantee}
Scalable $k$-means++ \cite{Bahmani} has empirically observed that their sampling distribution is robust to outliers. \cite{ChawlaG} suggests a simple modification of Lloyd's $k$-means method to make it robust to outliers. However, these methods do not provide any theoretical justification for their clustering quality.

  \begin{table*}
   \begin{center}
  \scalebox{0.83}{
    \begin{tabular}{|c|p{2.3cm}|c|c|p{1.5cm}|p{2.5cm}|}
   \hline
    Result                      &Approximation           &No. of clusters                &No. of outliers       &Practical       &Running time  \\       
                                &guarantee               &in the output                  &discarded             &algorithm       & \\       
\hline
    This       &$64+\epsilon$                     &$(1+c)k$                    &$z$               &Yes             &$O(ndk)$   \\
         paper                       &$O(\log k)$                     &$k$                          &$z$              &Yes             &$O(ndk) $   \\
    \hline
      \cite{tkmeanspp}       &$(64+\epsilon)$     &$(1+c)k$                         &$\frac{(1+c)(1+\epsilon)z}{c(1-\mu)}$                  &Yes              & $\tilde{O}(ndk) $    \\
     \cite{tkmeanspp}       &$O(\log k)$     &$k$                         &$O(z\log k)$                  &Yes              & $ \tilde{O}(ndk) $    \\
\hline
      \cite{DeshpandeKP20}       &$5$     &$kn/z$                         &$O(z)$                  &Yes              & $ {O}(ndk) $    \\
\hline
  \cite{KrishnaswamyLS18}       &$(53.002+\epsilon)$     &$k$                         &$z$                  &No              & $ n^{O( {1}/{\epsilon^3})} $    \\
  \cite{FriggstadKRS18}         &$(1+\epsilon)$          &$k(1+\epsilon)$             &$z$                  &No              & $kn^{d^{O(d)}}$          \\
   \cite{chen2008constant}      &$O(1)$                  & $k$                        &$z$                  &No              & $\poly(n,k)$\\
   \cite{Charikar}              &$4(1+1/\epsilon)$       & $k$                        &$(1+\epsilon)z$      &No              & $n^{O(1)}$\\
\hline
\end{tabular}
}
\caption{{Comparison with related work on $k$-means clustering with outliers, where $n $ is the  no. of data points, $d $ is the  dimension, $k $ is the  given no. of clusters, and $z $ is the    given no. of outliers. In \cite{tkmeanspp}, $c$ is a given parameter, and $\mu>0$ is an arbitrary constant. Their algorithm crucially requires an initial guess of the optimal clustering cost.  \cite{Charikar,chen2008constant} are about a closely related problem called $k$-median clustering with outliers.
}}\label{tab:comparison}
\end{center}
\end{table*}



\subsection{Practical algorithms with theoretical
guarantee}
Deshpande \textit{et. al.}~\cite{DeshpandeKP20} suggest a bi-criteria $O(1)$-approximation algorithm for $k$-means with $z$ outliers problem. They propose a sampling distribution consisting of a uniform and $D^2$ sampling mixture. They show that if $O(kn/z)$ points are picked using this distribution, the sampled points contain a set of $k$ points that gives $O(1)$-factor approximation while discarding slightly more points than the optimal number of outliers. The advantage of our proposal over~\cite{DeshpandeKP20} is that we do not discard any extra points as outliers.

Bhaskara \textit{et. al.}~\cite{tkmeanspp} suggest a bi-criteria approximation algorithm for the problem. They show that thresholding on the $D^2$ sampling distribution makes it robust to outliers. However,  their algorithms crucially require an initial guess of the optimal clustering cost over the inlier -- the very quantity we want to estimate. With this assumption, they showed the following two results. In the first, their algorithm output $(1+c)k$ centers and obtain $(64+\epsilon)$ approximations by discarding ${(1+c)(1+\epsilon)z}/{c(1-\mu)}$  outliers, where $c$ is some parameter, and $\mu$ is a constant. In the second result,   their algorithm outputs exactly $k$ centers and obtains $O(\log k)$ approximations by discarding $O(z\log k)$ outliers. In both cases, the running time of their algorithm is $\tilde{O}(ndk)$. Our algorithm compares with their results as follows:  ~\cite{tkmeanspp} requires a guess of the optimal clustering cost over inliers $(\rho_{X}(C_{\opt}))$ to compute their sampling distributions. Note that this is the quantity we want to estimate in $k$-means with outlier problems. In contrast, we do not require such a guess of the optimal clustering cost to compute our sampling distributions. Moreover, their algorithm discards more than $z$ outliers, whereas ours discards only $z$ points as outliers. We summarise our comparison with the related work in Table~\ref{tab:comparison}.

\section{Seeding algorithm for  $k$-means with outliers and its analysis}
In this section we  present our algorithm -- \texttt{Seeding algorithm for  $k$-means with outliers} and its theoretical guarantee. We first state (recall) some notation for convenience.  
 
We denote  $X \subseteq \R^{d}$ as the set of $n$ input points, integers $k, z > 0$  denote the number of cluster centers, and the number of outliers, respectively. 
 We denote by 
\begin{align}
  \phi_{A}(C) &= \sum_{x \in A} \min_{c \in C} \sqd{x-c},
\end{align}
the contribution of points in a subset $A \subseteq X$. Let $C_{\opt}$ be the set of optimal $k$ centers for the \emph{ $k$-means with outlier} problem and let $Y_{\opt}$ be the optimal subset of inliers, then 
\begin{align}
   \rho_X(C_{\opt})& = \phi_{Y_{\opt}}\left(C_{\opt}\right). \numberthis \label{eq:robust_cost1}
\end{align}
In the optimal solution, each point of $Y_{\opt}$ is assigned to its nearest center in $C_{\opt}$. This induces a natural partition of the inliers $Y_{\opt}$ as $A_{1} \cup A_{2} \cup \dotsb \cup A_{k}$ into disjoint subsets, with means $\mu_{1}, \mu_{2}, \dotsc, \mu_{k}$, respectively, while $X \setminus Y_{\opt}$ are the outliers. Therefore,
\begin{align}
    \rho_X(C_{\opt})& = \phi_{Y_{\opt}}\left(C_{\opt}\right) = \sum_{j=1}^{k} \phi_{A_{j}}\left(\{\mu_{j}\}\right).\numberthis \label{eq:robust_cost2}
     \end{align}

We present our algorithm -- \texttt{Seeding algorithm for  $k$-means with outliers} in  Algorithm~\ref{algo:robust}. Our sampling (stated in Line~\ref{step:sample} of Algorithm~\ref{algo:robust}) is a simple modification of $k$-means++ sampling distribution which is computed by taking minimum of $k$-means++ sampling distribution and the term ${\eta \cdot \rho_{X}(S_{i-1})}/{z}$, where $\eta$ is a parameter, $z$ is the number of outliers. Note that the term $\rho_{X}(S_{i-1})$  is the $k$-clustering cost (see Equations~\eqref{eq:robust_cost1},~\eqref{eq:robust_cost2}) by considering the $S_{i-1}$ sampled points as cluster centers, and  discarding the farthest $z$ as outliers. 

     \begin{algorithm}[H]
          \caption{Seeding algorithm for  $k$-means with outliers.} \label{algo:robust}
\begin{algorithmic}[1]
   \STATE {\bfseries Input:} a set $X \subseteq \R^{d}$ of $n$ points, number of outliers $z$, a parameter $\eta > 0$, and number of iterations $t$.
   \STATE {\bfseries Output:} a set $S \subseteq X$ of size $t$.
   \STATE Initialize $S_{0} \leftarrow \emptyset$.
   \FOR{$i=1$ {\bfseries to} $t$}
   \STATE Pick a point $x \in X$ from the following distribution: \\ 
   $\prob{\text{picking}~ x} \propto \min\left\{ \phi_{\{x\}}(S_{i-1}),  \dfrac{\eta \cdot \rho_{X}(S_{i-1})}{z}\right\}$ \label{step:sample} \\
   \STATE $S_{i} \leftarrow S_{i-1} \cup \{x\}$
   \STATE $i \leftarrow i+1$
   \ENDFOR
   \STATE $S \leftarrow S_{t}$
   \STATE \textbf{return} $S$
\end{algorithmic}
      \end{algorithm}
We note that a large part of our analysis is similar to Bhaskara et al. \cite{tkmeanspp}. However, the key difference is that they need an estimate of $\rho_{X}(C_{\opt})$ to calculate their sampling probabilities -- the optimal clustering cost over inliers, the very quantity which we want to estimate.  On the other hand, our sampling scheme does not require any estimate of $\rho_{X}(C_{\opt})$ to calculate the sampling probabilities.


Another difference from Bhaskara et al. \cite{tkmeanspp} is that we can set our  parameter $\eta = 1$, which allows us to discard exactly $z$ points as outliers, the same as what the optimal solution discards. Whereas their algorithm cannot discard exactly $z$ points as outliers because they need to estimate $\rho_{X}(C_{\opt})$ which is required to compute the probability distribution for sampling the points.

\noindent\textbf{Overview of analysis techniques:} Our algorithm is a simple modification of~\cite{AV2007}, where we perform a   thresholding of the $k$-means++ sampling distribution. This thresholding is controlled by the parameter $\eta$ (Line~\ref{step:sample} of Algorithm~\ref{algo:robust}) that ensures that a small number of outliers points are sampled. In our analysis, \begin{inparaenum} [(i) ]\item we monitor the number of sampled points, and the inlier clusters determined by them (by marking the farthest $z$ points as outliers), \item the number of so-called \texttt{wasted} iteration due to sampling outliers.\end{inparaenum}  We measure these quantities using different potential functions.  We show two theoretical guarantees  on our algorithm depending on the number of points sampled in Algorithm~\ref{algo:robust}. We state it in the following theorem.


\begin{thm}\label{thm:robust}
For any constant parameter $\eta \geq 1$, Algorithm \ref{algo:robust}, with  probability at least $\delta$, satisfies the following guarantee:
\begin{itemize}
    \item $\expec{\rho_{X}(S_{k})} = O(\log k) \rho_{X}(C_{\opt}),$ \qquad \text{when the number of iterations~} $t=k$.
    \item 
$\rho_{X}(S_{t}) \leq \frac{(\eta + 64) (1+c) \rho_{X}(C_{\opt})}{(1-\delta) c}$, \qquad \text{when the number of iterations~} $t = (1+c)k$.
\end{itemize}
\end{thm}



We start with proving the following lemma. 
\begin{lem} \label{lemma:conditions}
Suppose after iteration $i$ of Algorithm \ref{algo:robust}, we satisfy the following two conditions:
\[
\rho_{X}(S_{i}) > \alpha \rho_{X}(C_{\opt}) \quad \text{and} \quad \sum_{x \in X} \min\left\{\phi_{\{x\}}(S_{i}), \dfrac{\eta \rho_{X}(S_{i})}{z}\right\} \leq \gamma \rho_{X}(C_{\opt}).
\]
Then there exists $Y \subseteq X$ such that $\phi_{Y}(S_{i}) \leq \gamma \rho_{X}(C_{\opt})$ and $\size{X \setminus Y} \leq \gamma z/\alpha \eta$.
\end{lem}
\begin{proof}
Let $Y = \{x \suchthat \phi_{\{x\}}(S_{i}) \leq \eta \rho_{X}(S_{i})/z\}$. Then we get
\[
\sum_{x \in X} \min\left\{\phi_{\{x\}}(S_{i}), \dfrac{\eta \rho_{X}(S_{i})}{z}\right\} = \phi_{Y}(S_{i}) + \size{X \setminus Y} \frac{\eta \rho_{X}(S_{i})}{z} \leq \gamma \rho_{X}(C_{\opt}).
\]
Thus, $\phi_{Y}(S_{i}) \leq \gamma \rho_{X}(C_{\opt})$ and $\size{X \setminus Y} \leq \gamma z/\alpha \eta$, using $\rho_{X}(S_{i}) > \alpha \rho_{X}(C_{\opt})$.
\end{proof}

Suppose only the second condition is satisfied after iteration $i$, then we can choose any $\alpha$ and still get that either $\rho_{X}(S_{i}) \leq \alpha \rho_{X}(C_{\opt})$ or otherwise, using both the conditions in Lemma \ref{lemma:conditions}, we have $\phi_{Y}(S_{i}) \leq \gamma \rho_{X}(C_{\opt})$ for some $Y \subseteq X$ with $\size{X \setminus Y} \leq \gamma z/\alpha \eta$. This implies $\max\{\alpha, \gamma\}$ approximation guarantee while discarding at most $\gamma z/\alpha \eta$ points as outliers instead of $z$. In particular, we can use $\alpha = \gamma$ and $\eta = 1$ to get $\alpha$-approximation while discarding at most $z$ points as outliers.

We now focus on proving that the second condition is satisfied in expectation. First, we show that after $t=k$ iterations, we get $\gamma = O(\log k)$. 
\begin{thm} \label{thm:log-k}
After $k$ iterations of Algorithm \ref{algo:robust}, we get
\[
\expec{\sum_{x \in X} \min\left\{\phi_{\{x\}}(S_{k}), \dfrac{\eta \rho_{X}(S_{k})}{z}\right\}} = O(\log k) \cdot \rho_{X}(C_{\opt}).
\]
\end{thm}

Consider any optimal inlier cluster $A$. Below we show that if $\phi_{A}(S_{i}) \geq 64~\phi_{A}(\{\mu\})$ then $\phi_{\{\mu\}}(S_{i}) \geq \frac{31}{\size{A}} \phi_{A}(\{\mu\})$ and there exists a large subset $B \subseteq A$ such that $\phi_{\{x\}}(S_{i})$ values for all $x \in B$ are within a small constant factor of each other. In other words, $D^{2}$ sampling or sampling w.r.t. $\phi_{\{x\}}(S_{i})$ has an approximately uniform probability distribution over $B$.
\begin{lem} \label{lemma:approx-unif}
Let $A \subseteq X$ be any subset of points with mean $\mu$. Suppose $\phi_{A}(S_{i}) \geq 64~\phi_{A}(\{\mu\})$. Then $S_{i}$ satisfies the following two properties:
\begin{enumerate}
\item $64~\phi_{A}(\{\mu\}) \leq \phi_{A}(S_{i}) \leq \dfrac{64}{31} \size{A} \phi_{\{\mu\}}(S_{i})$, therefore, $\phi_{\{\mu\}}(S_{i}) \geq \dfrac{31}{\size{A}} \phi_{A}(\{\mu\})$.
\item Let $B = \left\{x \in A \suchthat \dfrac{\phi_{\{\mu\}}(S_{i})}{3} \leq \phi_{\{x\}}(S_{i}) \leq \dfrac{7~\phi_{\{\mu\}}(S_{i})}{3}\right\}$. Then, $B$ is a reasonably large subset of $A$, i.e., $\size{B} \geq \dfrac{25}{31} \size{A}$.
\end{enumerate}
\end{lem}
\begin{proof}
For any $x \in A$ and any $s \in S_{i}$, the triangle inequality gives 
\[
\frac{1}{2}~\norm{\mu - s}^{2} - \norm{x - \mu}^{2} \leq \norm{x - s}^{2} \leq 2 \norm{x - \mu}^{2} + 2 \norm{\mu - s}^{2}.
\]
Thus, $\frac{1}{2}~\phi_{\{\mu\}}(S_{i}) - \norm{x - \mu}^{2} \leq \phi_{\{x\}}(S_{i}) \leq 2~\phi_{\{\mu\}}(S_{i}) + 2 \norm{x - \mu}^{2}$. Summing the right-hand inequality over $x \in A$, we get 
\[
\phi_{A}(S_{i})  \leq 2 \size{A} \phi_{\{\mu\}}(S_{i}) + 2~\phi_{A}(\{\mu\}) \leq 2 \size{A} \phi_{\{\mu\}}(S_{i})+ \frac{1}{32}~ \phi_{A}(S_{i}),
\]
which gives the first property $\phi_{A}(S_{i}) \leq \frac{64}{31} \size{A} \phi_{\{\mu\}}(S_{i})$. Now let $B' = \{x \in A \suchthat \norm{x - \mu}^{2} \leq \frac{1}{6}  \phi_{\{\mu\}}(S_{i})\}$. Then by using triangle inequality for squared norms, we can check that $B' \subseteq B$. Since $\sum_{x \in A} \norm{x - \mu}^{2} = \phi_{A}(\{\mu\})$, Markov's inequality implies that $\size{A \setminus B'} \leq 6~ \phi_{A}(\{\mu\})/\phi_{\{\mu\}}(S_{i})$. Now using the first property $\phi_{A}(S_{i}) \leq \dfrac{64}{31} \size{A} \phi_{\{\mu\}}(S_{i})$ and the assumption $\phi_{A}(S_{i}) \geq 64~\phi_{A}(\{\mu\})$, we get $\size{A \setminus B'} \leq 6 \size{A}/31$. Since $B' \subseteq B$, we get the second property $\size{B} \geq \frac{25}{31} \size{A}$.
\end{proof}

The lemma below shows that in each iteration of Algorithm \ref{algo:robust}, if we pick a point from an optimal inlier cluster $A$, then we get a $64$-approximation guarantee for it, in expectation.
\begin{lem} \label{lemma:64-approx}
For any $A \subseteq X$ and its mean $\mu$, the point $x$ picked by Algorithm \ref{algo:robust} in a single iteration satisfies $\expec{\phi_{A}(S_{i} \cup \{x\}) \given x \in A} \leq 64~\phi_{A}(\{\mu\})$.
\end{lem}
\begin{proof}
Let $B = \left\{x \in A \suchthat \dfrac{\phi_{\{\mu\}}(S_{i})}{3} \leq \phi_{\{x\}}(S_{i}) \leq \dfrac{7~\phi_{\{\mu\}}(S_{i})}{3}\right\}$, as defined in Lemma \ref{lemma:approx-unif}. Suppose $\frac{7}{3} \phi_{\{\mu\}}(S_{i}) \leq \eta \rho_{X}(S_{i})/z$. Then for all $x \in B$, we have $\min\{\phi_{\{x\}}(S_{i}), \eta \rho_{X}(S_{i})/z\} \geq \frac{1}{3} \phi_{\{\mu\}}(S_{i})$. Hence,
\begin{align*}
\sum_{x \in A} \min\{\phi_{\{x\}}(S_{i}), \eta \rho_{X}(S_{i})/z\} & \geq \sum_{x \in B} \min\{\phi_{\{x\}}(S_{i}), \eta \rho_{X}(S_{i})/z\} \\
& \geq \size{B} \frac{1}{3} \phi_{\{\mu\}}(S_{i}) \\
& \geq \frac{25}{31} \size{A} \cdot \frac{1}{3} \cdot \frac{31}{64} \frac{\phi_{A}(S_{i})}{\size{A}} \qquad \text{using Lemma \ref{lemma:approx-unif}} \\
& \geq \frac{\phi_{A}(S_{i})}{8}.
\end{align*}
Therefore, for all $x \in A$, we have 
\[
\frac{\min\{\phi_{\{x\}}(S_{i}), \eta \rho_{X}(S_{i})/z\}}{\sum_{x \in A} \min\{\phi_{\{x\}}(S_{i}), \eta \rho_{X}(S_{i})/z\}} \leq \frac{\phi_{\{x\}}(S_{i})}{8~\phi_{A}(S_{i})}.
\]
In other words, if we sample $x \in A$ as in Algorithm \ref{algo:robust} with probability proportional to $\min\{\phi_{\{x\}}(S_{i}), \eta \rho_{X}(S_{i})/z\}$ then, up to a multiplicative factor of $8$, we get a similar upper bound on any expectation calculated for sampling with probability proportional to $\phi_{\{x\}}(S_{i})$, i.e., $D^{2}$ sampling. Now the proof of Lemma \ref{lemma:64-approx} is immediate using the key lemma about $D^{2}$ sampling from Arthur and Vassilvitskii \cite{AV2007} stated below.
\end{proof}

The following key lemma was used by Arthur and Vassilvitski \cite{AV2007} in their analysis of $k$-means++ algorithm that uses $D^{2}$ sampling.
\begin{lem} \label{lemma:AV}
For any $A \subseteq X$ and its mean $\mu$, the point $x$ picked by $D^{2}$ sampling satisfies $\expec{\phi_{A}(S_{i} \cup \{x\}) \given x \in A} \leq 8~\phi_{A}(\{\mu\})$.
\end{lem}

Proof of Theorem \ref{thm:log-k} is then similar to that in Theorem $3.1$ of Bhaskara et al. \cite{tkmeanspp}. We  define $U_{i}$ as the number of points in optimal inlier clusters untouched
 by Algorithm \ref{algo:robust} until step $i$.  Further,  let $H_i$ denote the union of points in the covered (optimal) clusters until step $i$. Let $w_{i}$ be the number of wasted iterations, i.e., iterations that pick either outliers $X \setminus Y_{\opt}$ or repeat points from already touched inlier clusters $A_{j}$'s from the optimal solution $Y_{\opt} = A_{1} \cup A_{2} \cup \dotsc \cup A_{k}$. We denote $n_i$ as the number of uncovered optimal clusters at iteration $i$. 
 
  For brevity we denote 
\[
\xi(U_i, S_i):=\sum_{x \in U_i} \min\left\{\phi_{\{x\}}(S_{i}), \dfrac{\eta\cdot \rho_{X}(S_{i})}{z}\right\}.
\]
 We define a similar potential function as in Bhaskara et al. \cite{tkmeanspp}.
  

\[
\Psi_i=\frac{w_i \cdot \xi(U_i, S_i)}{n_i}.
\]

For any $i>0$, we have the following lemma:
\begin{lem}[Adapted from Lemma $9$ of \cite{tkmeanspp}]\label{lem:64_approx_potential}
$$\expec{\xi(H_i, S_i)}\leq 64.\rho_{X}(C_{\opt}).$$
\end{lem}
A proof of the above lemma is similar to that of Lemma~$9$ of \cite{tkmeanspp}, and follows from Lemma~\ref{lemma:64-approx} and \ref{lemma:approx-unif} along with the inductive argument. We defer it to the full version of the paper.

\begin{lem}[Adapted from Lemma $8$ of \cite{tkmeanspp}]\label{lem:difference_potential} Let $S_i$ be the set of sampled points in the $i$-th iteration, then we have
\[
\expec{\Psi_{i+1}-\Psi_i|S_i}\leq \frac{\alpha\cdot \rho_{X}(C_{\opt}) + \xi(H_i, S_i)}{k-i}.
\]
\end{lem}
In Lemma~\ref{lem:difference_potential},  $\alpha$ is a constant mentioned as in Lemma~\ref{lemma:conditions}.   A proof of Lemma~\ref{lem:difference_potential} is analogous to the proof of Lemma $8$ of~\cite{tkmeanspp}. We defer it to the full version of the paper. 

We now conclude a proof of Theorem~\ref{thm:log-k}.  Its proof completes via combining Lemma~\ref{lem:64_approx_potential} and~\ref{lem:difference_potential}, and summing over $0\leq i \leq k-1$.


Further, similar to Bhaskara et al. \cite{tkmeanspp}, we can show the following generalization for $t = (1+c)k$ iterations.
\begin{thm}
For any $\delta > 0$ and parameters $c, \eta > 0$, after $t = (1+c) k$ iterations of Algorithm \ref{algo:robust} we get
\[
\sum_{x \in X} \min\left\{\phi_{\{x\}}(S_{t}), \dfrac{\eta \rho_{X}(S_{t})}{z}\right\} \leq \frac{(\eta + 64) (1+c) \rho_{X}(C_{\opt})}{(1-\delta) c},
\]
with probability at least $\delta$.
\end{thm}

\section{Experiments}\label{sec:exp}
\paragraph{Hardware description.} We performed our experiments on a machine having the following configuration: 
CPU: Intel(R) Core(TM) i5-3320M CPU @ 2.70GHz x 4; Memory: 8 GB. 

\paragraph{Baseline algorithms:} We study the performance of our algorithm (Algorithm~\ref{algo:robust}) to find $k$ initial cluster centers with the following baselines:
\begin{inparaenum}[(a)]
\item~$\tkmpp$ (Greedy Sampling for Approximate Clustering in the Presence of Outliers)~\cite{tkmeanspp},
\item $\rkmpp$ (Robust $k$-means++~\cite{DeshpandeKP20}), 
\item $\kmpp$ ($k$-means++ \cite{AV2007}) 
, and \item random seeding~\cite{Lloyd}. 
\end{inparaenum}
Robust $k$-means++ ($\rkmpp$)~\cite{DeshpandeKP20} uses $(\alpha, 1-\alpha)$ mixture of uniform and $D^2$ sampling distribution. For a parameter $\delta \in (0, 1)$, their algorithm samples $O(k/\delta)$ points, then  $k$ points are picked from this sampled set using \textit{weighted $k$-means++} sampling~\cite{DeshpandeKP20,Bahmani}. We use $\alpha=1/2$ and $\delta=0.1$ for all the experiments. $\tkmpp$~\cite{tkmeanspp} requires one parameter -- an initial guess of optimal clustering cost -- to derive the probability distribution on each data point. In their paper they did not mention any principal way of guessing the clustering cost. For our empirical comparison we used, the cost $k$-means++ results as an initial guess. Similar to our algorithm, they also require an error parameter to derive the  probability distribution. In order to have a fair comparison, we used the same value of the error  parameter $\beta$ in both the algorithms.

\paragraph{Evaluation metric:}
In all the baselines, once we sample   the $k$ cluster centers, we mark the farthest $z$ points as outliers.
  We then note the clustering cost over inliers considering the $k$   sampled points as the cluster centers. 
We note the minimum, maximum, and average values  of the clustering cost are over $10$ repetitions. We also use \textit{precision} and \textit{recall} as our evaluation metric. If $z^*$ is the set of  true outliers and $z$  is the set of outliers reported by the algorithm, then \textit{precision}:=$|z^*\cap z|/|z|$, and \textit{recall}:=$|z^*\cap z|/|z^*|$. We also note the seeding time of the baseline algorithms to have a comparison of their efficiency.  

\begin{table*}[ht]
 \begin{center}
  \caption{Result on \texttt{Synthetic} dataset with $k=20$ and $z=25,50, 100$  outliers.  We mark the farthest  $25, 50$ and $100$ points as outliers, respectively. For $\rkmpp$ we use $\delta=0.1$ and $\alpha=1/2$. 
  We used the parameter $\beta=0.5$ for both $\tkmpp$ and our method. All costs are multiplicative of $10^4$.
}\label{tab:Synthetic}

 \scalebox{0.85}{
 \begin{tabular}{SSSSSSSSS SSS }
    \toprule
    {Method}& {$z$}&\multicolumn{3}{c}{Precision}&\multicolumn{3}{c}{Recall} &\multicolumn{3}{c}{Cost}&{Time (s)}  \\%
           & & {Max.} & {Avg.} & {Med.} & {Max.} & {Avg.} & {Med.}  & {Min.} & {Avg.} & {Med.} &{} \\ 
               \midrule
             $\rand$  &25   &0.52     &0.31      &0.30    &0.52     &0.31      &0.30     &0.57    &1.20     &1.05  &0.06 \\     
              $\kmpp$  &25   &0.72     &0.53      &0.50     &0.72     &0.53      &0.5    &\textbf{~0.22}     &0.43     &0.41     & 0.20\\  
               $\tkmpp$  &25    &0.88     &0.67      &0.72     &0.88    &0.67     &0.72     &0.25         &0.58 &0.51     &0.92\\  
                 $\rkmpp$  &25    &0.88     &0.82      &\textbf{~0.84}     &0.88    &0.82     &\textbf{~0.84}     &0.25         &0.29 &\textbf{~0.30}     &6.5\\ 
               \hline
              \text{This work}  &25    &\textbf{~0.92}     &\textbf{~0.88}      &\textbf{~0.84}     &\textbf{~0.92}    &\textbf{~0.88}     &\textbf{~0.84}     &\textbf{~0.22}         &\textbf{~0.26} &{~0.32}     &0.90\\  
                 \midrule
                 $\rand$  &50   &0.34     &0.18      &0.21    &0.34     &0.18      &0.21     &1.42    &2.65     &2.54  &0.10 \\     
              $\kmpp$  &50   &0.80     &0.42      &0.38     &0.80     &0.42      &0.38    &0.31     &0.89     &0.85     & 0.23\\  
               $\tkmpp$  &50    &0.82     &0.76      &0.48     &0.82    &0.76     &0.48     &0.34         &0.37 &0.69     &0.49\\  
                $\rkmpp$  &50    &0.88     &0.82      &0.83     &0.88    &0.82     &0.83     &\textbf{~0.18}         &\textbf{~0.19} &\textbf{~0.18}     &6.5\\
               \hline
              \text{This work}  &50    &\textbf{~0.92}     &\textbf{~0.88}      &\textbf{~0.84}     &\textbf{~0.92}    &\textbf{~0.88}     &\textbf{~0.84}     &{~0.22}         &0.26 &{~0.32}     &0.60\\
              \hline
               $\rand$  &100   &0.36     &0.20      &0.20     &0.44    &0.24     &0.24     &0.95         &1.71 &1.85     &0.06\\     
              $\kmpp$  &100   &0.68     &0.49      &0.51     &0.82    &0.59     &0.62     &0.25         &0.58 &0.55     &0.19\\ 
               $\tkmpp$  &100   &0.75     &0.49      &0.54     &0.90         &0.59     &0.65  &0.23 &0.71      &0.60       &0.71\\ 
                $\rkmpp$  &100   &\textbf{~0.81}     &\textbf{~0.78}      &\textbf{~0.77}     &\textbf{~0.98}         &\textbf{~0.94}     &\textbf{~0.93}  &0.32 &0.42      &0.41       &6.5\\
                               \hline
   \text{This work}  &100   &{~0.77}     &{~0.76}      &{~0.74}     &{~0.93}    &{~0.91}     &{~0.89}     &\textbf{~0.18}         &\textbf{~0.22} &\textbf{~0.21}     &0.93\\    
             
      \bottomrule
  \end{tabular}
  }
\end{center}
\end{table*}
 \subsection{Results on Synthetic Data Sets}
\paragraph{Dataset.} We generate synthetic dataset in the similar way as used in  $k$-means++ \cite{AV2007}. 
We discuss it as follows.
 We pick $k + z$ uniformly random points from a 
large $d$-dimensional hypercube of side length $s = 100$. We use $k$ points from them as means and pick $n/k$ points around 
each of them from a random Gaussian of unit variance. This gives a data set of $n+z$ points with $n$ points clustered into $k$ 
clusters and the remaining $z$ as outliers. 
 
\paragraph{Empirical Evaluation.} 
We  perform experiments on synthetic datasets with values $n=1000, d=2, k=20$ and the number of outliers $25, 50, 100$, and we summarise our results in Table~\ref{tab:Synthetic}.


\paragraph{Insight.} In almost every scenario our algorithm outperforms with respect to  random, $k$-means++ and $\tkmpp$ in terms of both precision/recall, and clustering cost metric.  
Our performance was comparable/better on most of the instances with respect to  $\rkmpp$. 
Our algorithm is much faster than $\rkmpp$, and  comparable with respect to $\tkmpp$, however it was slightly off with respect to  random, $k$-means++.


\subsection{Results on real world data sets}
\subsubsection{Small clusters as outliers:} 
In this setting, we consider a few very small clusters as outliers and try to locate them using our baselines.


  \paragraph{Results on  \texttt{Shuttle}  dataset.} \texttt{Shuttle} training data set from UCI Machine Learning Repository \cite{UCI2013} contains $43,500$ points. It has $7$ classes in total. The two smallest classes contain only $17$ points and we would like to detect these as outliers. 
 We run our baselines  on the \texttt{Shuttle} dataset with $k\in \{5, 10, 15\}$.  
We summarise our empirical findings in Table~\ref{tab:shuttle_small}.

\begin{table*}[ht]
 \begin{center}
  \caption{Result on \texttt{Shuttle} dataset  considering  the two smallest classes as outliers which contain only $17$ points.
  We mark the farthest $21, 34$, and  $51$ points as outliers when the values of $k$ are $5,10,$ and  $15$, respectively. In both $\tkmpp$ and our algorithm, we used the error parameter  $\beta=0.1$. For $\rkmpp$ we use $\delta=0.1$ and $\alpha=1/2$. 
 All costs are multiplicative of $10^8.$
}\label{tab:shuttle_small}
\scalebox{0.85}{
 \begin{tabular}{SSSSSSSSS SSS }
 
    \toprule
    {Method}& {$k$}&\multicolumn{3}{c}{Precision}&\multicolumn{3}{c}{Recall} &\multicolumn{3}{c}{Cost}&{Time (s)}  \\%
           & & {Max.} & {Avg.} & {Med.} & {Max.} & {Avg.} & {Med.}  & {Min.} & {Avg.} & {Med.} &{} \\ 
               \midrule
             $\rand$  &5   &0.19     &\textbf{~0.19}      &0.19     &{0.23}    &\textbf{~0.23}     &0.23     &3.56         &3.66 &3.63     &1.01\\     
              $\kmpp$  &5   &\textbf{~0.28}     &0.17      &\textbf{~0.23}     &\textbf{~0.35}    &0.21     &\textbf{~0.29}     &1.77         & 2.22 &2.27     &2.26\\ 
               $\tkmpp$  &5   &0.23     &0.17      &\textbf{~0.23}     &0.29    &0.21     &\textbf{~0.29}     &1.70         &2.26 &2.23     &3.42\\ 
               $\rkmpp$  &5   &0.26     &0.18      &\textbf{~0.23}     &0.33    &\textbf{~0.23}     &\textbf{~0.29}     &1.85         &\textbf{~2.07} &\textbf{~2.01}     &5.3\\ 
               \hline
              \text{This work}  &5   &0.23     &\textbf{~0.19}      &\textbf{~0.23}     &0.29    &\textbf{~0.23}     &\textbf{~0.29}     &\textbf{~1.53}         &{2.08} &\textbf{~2.01}     &3.54\\ 
                 \midrule
                $\rand$  &10   &0.14     &0.14      &0.14     &0.29    &0.29     &0.29     &1.55         &1.78 &1.80     &1.54\\     
              $\kmpp$  &10   &0.26     &\textbf{~0.16}      &\textbf{~0.17}     &0.52    &0.31     &\textbf{~0.35}     &0.72         &0.95 &0.92     &3.41\\ 
              
               $\tkmpp$  &10   &0.26     &0.14      &0.11     &0.52    &0.29     &0.23     &0.73         &0.96 &0.90     &7.35\\ 
                $\rkmpp$  &10   &0.34     &0.15      &0.15     &0.68    &0.30     &0.30     &0.69         &0.85 &0.85     &{~~11.5}\\ 
               \hline
               \text{This work}  &10   &\textbf{~0.35}     &\textbf{~0.16}      &\textbf{~0.17}     &\textbf{~0.70}    &\textbf{~0.32}     &\textbf{~0.35}     &\textbf{~0.56}         &\textbf{~0.65} &\textbf{~0.64}     &{~8.7}\\ 
                   \midrule
                $\rand$  &15   &0.13     &0.13      &0.13     &0.41    &0.41     &0.41             &0.80 &0.85  &0.84  &2.15\\     
              $\kmpp$  &15   &0.25     &0.17      &{0.19}     &0.76    &0.52     &0.58     &0.39         &0.49 &0.50     &5.25\\ 
              
               $\tkmpp$  &15   &0.25     &0.17      &{0.19}     &0.76    &0.52     &{0.58}     &0.36         &0.48 &0.48     &{~11.20}\\ 
                $\rkmpp$  &15   &0.29     &\textbf{~0.21}      &\textbf{~0.22}     &0.88    &0.64     &\textbf{~0.67}     &0.42         &0.49 &0.49     &{~19.5}\\ 
               \hline
               \text{This work}  &15   &\textbf{~0.31}     &{0.18}      &0.18     &\textbf{~0.94}    &\textbf{~0.55}     &{0.52}     &\textbf{~0.27}         &\textbf{~0.31} &\textbf{~0.31}     &{~13.12}\\ 
      \bottomrule
  \end{tabular}
}
 
  \end{center}
\end{table*}

\paragraph{Results on   \texttt{KDDCup Full} dataset:} 
\texttt{KDDFull}~\cite{UCI2013} This dataset is from $1999$ kddcup competition and contains instances describing
connections of sequences of tcp packets, and have about $4.9$M data points. We only consider the $34$ numerical features of this dataset. We also normalize each feature so that it has zero mean and unit standard deviation. 
There are $23$ classes in this dataset, $98.3\%$ points of
the dataset belong to $3$ classes (\texttt{normal} $19.6\%$, \texttt{neptune} $21.6\%$, and \texttt{smurf} $56.8\%$). We consider the remaining small clusters as outliers,  and the  number of outliers  is $45747$.
We run  our baselines on the \texttt{KDDFull} dataset with $k=3, 5$ and considering above mentioned $45747$ points as outliers.
 We summarise our empirical findings in Table~\ref{tab:kdd}.
 
\begin{table*}[ht]
 \begin{center}
  \caption{Result on \texttt{KDDCup} dataset with $k=3$ and $k=5$ with $45747$ outliers. We used the error parameter  $\beta=0.1$ for both  $\tkmpp$ and our algorithm. For $\rkmpp$ we use $\delta=0.1$ and $\alpha=1/2$. 
  All costs are multiplicative of $10^7.$}\label{tab:kdd}
 \scalebox{0.85}{
\begin{tabular}{SSSSSSSSS SSS }
 
    \toprule
    {Method}& {$k$}&\multicolumn{3}{c}{Precision}&\multicolumn{3}{c}{Recall} &\multicolumn{3}{c}{Cost}&{Time (s)}  \\%
           & & {Max.} & {Avg.} & {Med.} & {Max.} & {Avg.} & {Med.}  & {Min.} & {Avg.} & {Med.} &{} \\ 
               \midrule
             $\rand$  &3   &0.64     &\textbf{~0.63}      &\textbf{~0.64}     &0.64    &\textbf{~0.63}     &\textbf{~0.64}     &{2.83}         &{4.04} &4.10     &{~77.1}\\     
              $\kmpp$  &3   &0.64     &0.59      &0.61     &0.64    &0.59     &0.61     &3.29         &5.41 &4.86     &{~119.2}\\  
               $\tkmpp$  &3   &0.64     &0.57      &0.61     &0.64    &0.57     &0.61     &2.83         &5.76 &5.83     &{~211.86}\\ 
            $\rkmpp$  &3   &0.64     &0.62      &0.61     &0.64    &0.62     &0.61     &\textbf{~2.8}         &\textbf{~2.8} &\textbf{~2.8}     &{~295}\\   
               \hline
              \text{This work}  &3   &\textbf{~0.65}     &\textbf{~0.63}      &\textbf{~0.64}     &\textbf{~0.65}    &\textbf{~0.63}     &\textbf{~0.64}     &\textbf{~2.83}         &4.20 &{4.07}     &{~225.23}\\  
                 \midrule
                $\rand$  &5  &0.64     &0.57      &\textbf{~0.60}     &0.64    &0.57     &\textbf{~0.60}     &\textbf{~2.48}         &4.23 &2.95     &{~100.2}\\      
              $\kmpp$  &5   &0.63     &0.54      &0.59     &0.63    &0.54     &0.59     &2.61         &\textbf{~3.27} &\textbf{~2.83}     &{~292.1}\\   
              
               $\tkmpp$  &5   &0.63     &0.45      &0.44     &0.63    &0.45     &0.44     &2.98         &3.85 &3.63     &{~402.3}\\  
                 $\rkmpp$  &5   &0.64     &0.58      &0.61     &0.64    &0.58     &0.61     &2.52         &3.8 &3.7     &{600.2}\\  
               \hline
               \text{This work}  &5  &\textbf{~0.65}     &\textbf{~0.60}      &\textbf{~0.60}     &\textbf{~0.65}    &\textbf{~0.60}     &\textbf{~0.60}     &2.54         &3.78 &3.11     &{~423.1}\\  
              
      \bottomrule
  \end{tabular}
}
 
  \end{center}
\end{table*}

 \paragraph{Insight.}For both \texttt{KDDCup} and \texttt{Shuttle} datasets, we noticed  that  our clustering results in terms of both -- cost and precision/recall metric  are significantly better than that of $\tkmpp$, $k$-means++ and random seeding, whereas it is comparable \textit{w.r.t.} $\rkmpp$.  The running time of our algorithm is faster than $\rkmpp$ and $\rand$ whereas it is  comparable with respect to  the other remaining  baselines.

 \subsubsection{Randomly sampling a small fraction of  points and adding large Gaussian noise to them
}\label{sec:large_gaussian} 
In this setting, we randomly  sample a small fraction of points  from the datasets, and add large Gaussian noise to their dimensions. We consider these points as outliers and and try to locate them using our baselines. 

\paragraph{Results on   \texttt{Skin Segmentation}~\cite{UCI2013} and \texttt{Shuttle} dataset:}  The Skin Segmentation data is constructed over \textit{B, G, R} color space. \textit{Skin} and \textit{Non-skin} data points are generated using skin textures from face images of diversity of age, gender, and race people. The dataset consists of $245057$ points and each data point has $4$ attributes. We sample $2.5\%$ of points and add large Gaussian noise to their features. 
We summarise our empirical findings in Table~\ref{tab:skin_noise}.

\paragraph{Results on \texttt{Shuttle} dataset:} 
For the \texttt{Shuttle} dataset, we randomly sample $1000$ points from the dataset and add large Gaussian noise to their features.  We consider them as outliers and try to locate them using our baseline methods. We summarise our empirical findings in Table~\ref{tab:shuttle_noise}.


\begin{table*}[ht]
 \begin{center}
  \caption{ Result on \texttt{Skin} dataset with $k= 20,40, 50$. We randomly sample $2.5\%$ of points and add large Gaussian noise to their features. In both $\tkmpp$ and our algorithm, we used the error parameter  $\beta=0.1$.  For $\rkmpp$ we use $\delta=0.1$ and $\alpha=1/2$.  All costs are multiplicative of $10^8.$}\label{tab:skin_noise}
 \scalebox{0.85}{
\begin{tabular}{SSSSSSSSS SSS }
 
    \toprule
    {Method}& {$k$}&\multicolumn{3}{c}{Precision}&\multicolumn{3}{c}{Recall} &\multicolumn{3}{c}{Cost}&{Time (s)}  \\%
           & & {Max.} & {Avg.} & {Med.} & {Max.} & {Avg.} & {Med.}  & {Min.} & {Avg.} & {Med.} &{} \\ 
               \midrule
             $\rand$  &20   &0.15     &0.10      &0.10     &0.15    &0.10     &0.10     &1.35         &1.69 &1.69     &{~25.23}\\     
              $\kmpp$  &20   &0.20     &\textbf{~0.16}      &0.15     &0.20    &\textbf{~0.16}     &0.15     &1.15         &{1.23} &{1.23}     &{~45.08}\\   
               $\tkmpp$  &20   &0.09     &0.08      &0.08     &0.09    &0.08     &0.08     &1.98         &2.20 &1.91     &{~74.13}\\  
                $\rkmpp$  &20   &\textbf{~0.24}     &\textbf{~0.16}      &\textbf{~0.17}     &\textbf{~0.24}    &\textbf{~0.16}     &\textbf{~0.17}     &{1.15}         &\textbf{~1.20} &\textbf{~1.20}     &{~123.7}\\ 
               \hline
              \text{This work}  &20   &{0.23}     &\textbf{~0.16}      &\textbf{~0.17}     &{0.23}    &\textbf{~0.16}     &\textbf{~0.17}     &\textbf{~1.10}         &1.35 &1.35     &{~85.30}\\ 
                 \midrule
                $\rand$  &40  &0.20     &0.13      &0.13     &0.20    &0.13     &0.13     &0.77         &0.96 &0.95     &{~31.20}\\         
              $\kmpp$  &40   &{0.25}     &\textbf{~0.19}      &0.19     &{0.25}    &\textbf{~0.19}     &0.19     &0.65         &0.68 &0.69     &{~85.80}\\     
              
               $\tkmpp$  &40   &0.16     &0.12      &0.12     &0.16    &0.12     &0.12     &0.79         &0.87 &0.86     &{~125.12}\\     
               $\rkmpp$  &40   &\textbf{~0.26}     &\textbf{~0.19}      &0.19     &\textbf{~0.26}    &\textbf{~0.19}     &0.19     &\textbf{~0.63}         &0.68 &\textbf{~0.65}     &{~196.3}\\  
               \hline
               \text{This work} &40   &0.23     &\textbf{~0.19}      &\textbf{~0.20}     &0.23    &\textbf{~0.19}     &\textbf{~0.20}     &{0.64}         &\textbf{~0.66} &\textbf{~0.65}     &{~140.21}\\   
                 \midrule
                $\rand$  &50  &0.15     &0.09      &0.08     &0.15    &0.09     &0.08     &0.71         &0.79 &0.78     &{~48.53}\\         
              $\kmpp$  &50   &0.22     &0.18      &0.17     &0.22    &0.18     &0.17     &\textbf{~0.51}         &0.54 &\textbf{~0.53}     &{~138.23}\\     
              
               $\tkmpp$  &50   &0.23     &0.19      &0.18     &0.23    &0.19     &0.18     &0.60         &0.66 &0.65     &{~204.39}\\
               $\rkmpp$ &50   &\textbf{~0.29}     &\textbf{~0.22}      &{0.18}     &\textbf{~0.29}    &\textbf{~0.22}     &{0.18}     &\textbf{~0.51}         &\textbf{~0.53} &\textbf{~0.53}     &{~324.39}\\  
               \hline
               \text{This work} &50   &\textbf{~0.29}     &\textbf{~0.22}      &\textbf{~0.19}     &\textbf{~0.29}    &\textbf{~0.22}     &\textbf{~0.19}     &0.52         &\textbf{~0.53} &0.54     &{~220.39}\\   
              
      \bottomrule
  \end{tabular}
}
 
  \end{center}
\end{table*}

\begin{table*}[ht]
 \begin{center}
  \caption{ Results on \texttt{Shuttle} dataset by randomly sampling $1000$ points and adding large Gaussian noise to their features.  In both $\tkmpp$ and our algorithm, we used the error parameter  $\beta=0.1$.
 All costs are multiplicative of $10^8$.}\label{tab:shuttle_noise}
 \scalebox{0.85}{
\begin{tabular}{SSSSSSSSS SSS }
    \toprule
    {Method}& {$k$}&\multicolumn{3}{c}{Precision}&\multicolumn{3}{c}{Recall} &\multicolumn{3}{c}{Cost}&{Time (s)}  \\%
           & & {Max.} & {Avg.} & {Med.} & {Max.} & {Avg.} & {Med.}  & {Min.} & {Avg.} & {Med.} &{} \\ 
               \midrule
             $\rand$  &10   &0.80     &0.70      &0.73     &0.80    &0.70     &0.73     &\textbf{~0.21}         &\textbf{~0.27} &\textbf{~0.25}     &1.24\\     
              $\kmpp$  &10    &0.81     &0.76      &0.80     &0.81    &0.76     &0.80     &0.29         &0.45 &0.47     &3.32\\  
               $\tkmpp$  &10    &{0.82}     &0.76      &0.80     &{0.82}    &0.76     &0.80     &0.29         &0.45 &0.47     &5.36\\
                $\rkmpp$  &10    &\textbf{~0.83}     &\textbf{~0.81}      &{0.80}     &\textbf{~0.83}    &\textbf{~0.81}     &{0.80}     &0.28         &0.36 &0.36     &{~10.0}\\ 
               \hline
              \text{This work}  &10    &{0.82}     &\textbf{~0.81}      &\textbf{~0.81}     &{0.82}    &\textbf{~0.81}     &\textbf{~0.81}     &0.29         &0.38 &0.38     &6.32\\ 
                 \midrule
                $\rand$  &20   &0.80     &0.80      &0.80     &0.80    &0.80     &0.80     &\textbf{~0.21}         &\textbf{~0.20} &\textbf{~0.19}     &2.60\\     
              $\kmpp$  &20   &0.82     &0.72      &0.79     &0.82    &0.72     &0.79     &0.25         &0.28 &0.26     &7.24\\ 
              
               $\tkmpp$  &20   &0.83     &0.73      &\textbf{~0.81}     &0.83    &0.73     &\textbf{~0.81}     &0.24         &0.27 &0.25     &7.35\\ 
               $\rkmpp$  &20  &\textbf{~0.85}     &\textbf{~0.76}      &0.80     &\textbf{~0.85}    &\textbf{~0.76}     &0.80     &0.22         &0.23 &0.20     &{~17.4}\\ 
               \hline
               \text{This work}  &20  &{0.84}     &{0.75}      &0.79     &{0.84}    &{0.75}     &0.79     &0.23         &0.24 &0.25     &{~11.23}\\ 
                   \midrule
                $\rand$  &40   &0.80     &0.77      &0.78     &0.80    &0.77     &0.78             &\textbf{~0.12} &\textbf{~0.14}  &\textbf{~0.14}  &4.94\\     
              $\kmpp$  &40  &0.82     &0.79      &0.78     &0.82    &0.79     &0.78     &0.14         &0.17 &0.17     &5.25\\ 
              
               $\tkmpp$  &40   &{0.83}     &0.79      &0.79     &{0.83}    &0.79     &0.79     &0.14         &0.17 &0.17     &11.20\\ 
                 $\rkmpp$   &40   &\textbf{~0.84}     &\textbf{~0.80}      &\textbf{~0.82}     &\textbf{~0.84}    &\textbf{~0.80}     &\textbf{~0.82}     &0.13         &0.14 &0.15     &19.70\\ 
               \hline
               \text{This work}  &40   &{0.83}     &\textbf{~0.80}      &{0.80}     &{0.83}    &{0.80}     &\textbf{~0.80}     &0.14         &0.16 &0.15     &13.12\\ 
      \bottomrule
  \end{tabular}
  }
\end{center}
\end{table*}

 \paragraph{Insight.} For both \texttt{Skin} and \texttt{Shuttle} datasets, we noticed  that  our clustering results in terms of precision/recall metric  are significantly better than that of $\tkmpp$, $k$-means++ and random seeding on the majority of instances. However it is comparable with respect to $\rkmpp$.  The running time of our algorithm is faster than $\rkmpp$ and $\rand$ whereas it is comparable with respect to the remaining baselines.

\section{Concluding remarks and open questions}
We suggest an outlier robust seeding for the  $k$-means clustering problem. Our major contribution lies in developing a simple and intuitive tweak to the $k$-means++ sampling algorithm, which makes it robust to outliers. The running time of our method is $O(ndk)$, and offers $O(\log k)$ approximation guarantee. We also gave a bi-criteria approximation algorithm, where our algorithm samples slightly more than $k$ points as cluster centers, which consist of a set of $k$ points (as inlier cluster centers) that gives $O(1)$ approximation guarantee. We empirically evaluate our algorithm on synthetic as well as real-world datasets, and show that our proposal outperforms \textit{w.r.t.}  $k$-means++~\cite{AV2007}, random initialization, and $\tkmpp$~\cite{tkmeanspp} algorithm in both the metric a) precision/recall and b) $k$-means with outliers clustering cost (Equation~\ref{eq:robust_cost}). However the performance of our proposal remains comparable to $\rkmpp$~\cite{DeshpandeKP20}. The running time of our algorithm is significantly faster than the random seeding and $\rkmpp$ whereas it is noted to be comparable \textit{w.r.t.} $k$-means++~\cite{AV2007}, and $\tkmpp$~\cite{tkmeanspp}. Our work leaves the possibility of several open questions: -- improving the theoretical bounds of our proposal,  and extending it to other class clustering problems with outliers setting. Finally, given our proposal's simplicity, performance, and efficiency, we hope it will be adopted in practice.

\bibliography{ref}
\bibliographystyle{plain}

\end{document}